\documentclass{article} % For LaTeX2e
\usepackage{nips13submit_e,times,natbib}
\usepackage{amsmath,amssymb,amsthm,color}
\usepackage{tikz}
\usepackage{hyperref}
\usepackage{url}

\newcommand{\ca}{{\mathcal A}}

\newcommand{\cd}{{\mathcal D}}

\newcommand{\cg}{{\mathcal G}}
\newcommand{\ch}{{\mathcal H}}

\newcommand{\cp}{{\mathcal P}}

\newcommand{\cx}{{\mathcal X}}

\newcommand{\sign}{\textrm{sign}}

\newcommand{\inner}[1]{\langle #1 \rangle}

\DeclareMathOperator{\Err}{Err}
\DeclareMathOperator{\sym}{sym}
\DeclareMathOperator{\val}{Val}
\DeclareMathOperator{\poly}{poly}

\DeclareMathOperator*{\VC}{VC}

\newtheorem{theorem}{Theorem}[section]
\newtheorem{lemma}[theorem]{Lemma}

\newtheorem{counter-example}[theorem]{Counter example}
\newtheorem{proposition}[theorem]{Proposition}
\newtheorem{open question}[theorem]{Open question}
\newtheorem{corollary}[theorem]{Corollary}
\newtheorem{conjecture}[theorem]{Conjecture}

\title{More data speeds up training time in learning halfspaces over sparse vectors}

\author{
Amit Daniely \\
Department of Mathematics\\
The Hebrew University\\
Jerusalem, Israel 
\And
Nati Linial\\
School of CS and Eng.\\
The Hebrew University\\
Jerusalem, Israel 
\And
Shai Shalev-Shwartz \\
School of CS and Eng.\\
The Hebrew University\\
Jerusalem, Israel 
 }

\nipsfinalcopy % Uncomment for camera-ready version

\begin{document}

\maketitle

\begin{abstract}
  The increased availability of data in recent years has led several
  authors to ask whether it is possible to use data as a {\em
    computational} resource. That is, if more data is available,
  beyond the sample complexity limit, is it possible to use the extra
  examples to speed up the computation time required to perform the
  learning task?

  We give the first positive answer to this question for a {\em
    natural supervised learning problem} --- we consider agnostic PAC
  learning of halfspaces over $3$-sparse vectors in $\{-1,1,0\}^n$.
  This class is inefficiently learnable using $O\left(n/\epsilon^2\right)$ examples.
  Our main contribution is a novel, non-cryptographic, methodology for
  establishing computational-statistical gaps, which allows us to show
  that, under a widely believed assumption that refuting random
  $\mathrm{3CNF}$ formulas is hard, it is impossible to efficiently learn this class
  using only $O\left(n/\epsilon^2\right)$ examples. We further show that
  under stronger hardness assumptions, even $O\left(n^{1.499}/\epsilon^2\right)$
  examples do not suffice.  On the other hand, we show a new
  algorithm that learns this class efficiently using
  $\tilde{\Omega}\left(n^2/\epsilon^2\right)$ examples. This formally
  establishes the tradeoff between sample and computational complexity
  for a natural supervised learning problem.

\end{abstract}

\section{Introduction}
In the modern digital period, we are facing a rapid growth of available datasets in science and technology. In most computing tasks (e.g. storing and searching in such datasets), large datasets are a burden and require more computation. 
However, for learning tasks the situation is radically different. A simple observation is that more data can never hinder you from performing a task. If you have more data than you need, just ignore it!

A basic question is how to learn from ``big data''. The statistical
learning literature classically studies questions like ``how
much data is needed to perform a learning task?" or ``how does accuracy
improve as the amount of data grows?" etc. In the modern,
``data revolution era'', it is often the case that the amount of data
available far exceeds the information theoretic requirements.
We can wonder whether this, seemingly redundant data, can be used
for other purposes.  An intriguing question in this vein, studied
recently by several researchers
(\citep{DecaturGoRo98,Servedio00,ShalevShTr12,BerthetRi13,ChandrasekaranJo13}),
is the following

\begin{quotation}\label{quest}
{\em Question 1:} Are there any learning tasks in which more data, {\em beyond the information theoretic barrier}, can {\em provably} be leveraged to speed up computation time? 
\end{quotation}
The main contributions of this work are:
\begin{itemize}
\item Conditioning on the hardness of refuting random $\mathrm{3CNF}$ formulas, we give the first example of a {\em natural supervised learning problem} for which the answer to Question 1 is positive.
\item To prove this, we present a novel technique to establish
  computational-statistical tradeoffs in supervised learning
  problems. To the best of our knowledge, this is the first such a result that is not based on cryptographic primitives.
\end{itemize}

Additional contributions are non trivial efficient algorithms for
learning halfspaces over $2$-sparse and $3$-sparse vectors using $\tilde{O}\left(\frac{n}{\epsilon^2}\right)$ and $\tilde{O}\left(\frac{n^2}{\epsilon^2}\right)$ examples respectively.

The natural learning problem we consider is the task of learning the class of {\em halfspaces  over $k$-sparse vectors}. Here, the instance space is the {\em space of $k$-sparse vectors},
\[
C_{n,k}=\{x\in \{-1,1,0\}^n\mid |\{i\mid x_i\ne 0\}|\le k\} ~,
\]
and the hypothesis class is {\em halfspaces over $k$-sparse vectors,} namely
\[
\ch_{n,k} = \{h_{w,b} :C_{n,k}\to\{\pm 1\} \mid h_{w,b}(x)=
\sign(\inner{w,x} + b),w\in \mathbb R^n,b\in\mathbb R\} ~,
\]
where $\inner{\cdot,\cdot}$ denotes the standard inner product in
$\mathbb R^n$. 

We consider the standard setting of agnostic PAC learning, which models
the realistic scenario where the labels are not necessarily fully
determined by some hypothesis from $\ch_{n,k}$.  Note that in the
realizable case, i.e. when some hypothesis from $\ch_{n,k}$ has zero
error, the problem of learning halfspaces is easy even over
$\mathbb{R}^n$.

In addition, we allow improper learning (a.k.a. representation
independent learning), namely, the learning algorithm
is not restricted to output a hypothesis from $\ch_{n,k}$, but only
should output a hypothesis whose error is not much larger than the
error of the best hypothesis in $\ch_{n,k}$. This gives the learner
a lot of flexibility in choosing an appropriate representation of the
problem. This additional freedom to the learner makes it much
harder to prove lower bounds in this model. Concretely, it is not clear how
to use standard reductions from NP hard problems in order to establish lower bounds for improper learning (moreover, \cite{ApplebaumBaXi08} give evidence that such simple reductions do not exist). 

The classes $\ch_{n,k}$ and similar classes have been studied by
several authors (e.g. \cite{LongSe13}). They naturally arise in
learning scenarios in which the set of all possible features is very
large, but each example has only a small number of active
features. For example:
\begin{itemize}
\item {\em Predicting an advertisement based on a search query:} Here,
  the possible features of each instance are all English words, whereas
  the active features are only the set of words given in the query.
\item {\em Learning Preferences~\citep{HazanKaSh12}:} Here, we have
  $n$ players. A {\em ranking} of the players is a permutation
  $\sigma:[n]\to [n]$ (think of $\sigma(i)$ as the rank of the $i$'th
  player). Each ranking induces a {\em preference} $h_{\sigma}$ over
  the ordered pairs, such that $h_{\sigma}(i,j)=1$ iff $i$ is ranked
  higher that $j$. Namely,
\[
h_{\sigma}(i,j)=
\begin{cases} 1 & \sigma (i)>\sigma (j)\\
 -1 & \sigma (i)<\sigma (j) \end{cases}
\]
The objective here is to learn the class, $\cp_n$, of all possible preferences. The problem of learning preferences is related to the problem of learning $\ch_{n,2}$: if we associate each pair $(i,j)$ with the vector in $C_{n,2}$ whose $i$'th coordinate is $1$ and whose $j$'th coordinate is $-1$, it is not hard to see that $\cp_{n}\subset \ch_{n,2}$: for every $\sigma$, $h_{\sigma}=h_{w,0}$ for the vector $w\in \mathbb R^n$, given by $w_i=\sigma(i)$. Therefore, every upper bound for $\ch_{n,2}$ implies an upper bound for $\cp_{n}$, while every lower bound for $\cp_{n}$ implies a lower bound for $\ch_{n,2}$. 
Since $\VC(\cp_n)=n$ and $\VC(\ch_{n,2})=n+1$, the information theoretic barrier to learn these classes is $\Theta\left(\frac{n}{\epsilon^2}\right)$.
In \cite{HazanKaSh12} it was shown that $\cp_n$ can be {\em efficiently} learnt using $O\left(\frac{n\log^3(n)}{\epsilon^2}\right)$ examples. In section \ref{sec:upper_bounds}, we extend this result to $\ch_{n,2}$.
\end{itemize}

We will show a positive answer to Question 1 for the class
$\ch_{n,3}$. To do so, we show\footnote{In fact, similar results hold
  for every constant $k \ge 3$. Indeed, since $\ch_{n,3}
  \subset \ch_{n,k}$ for every $k \ge 3$, it is trivial
  that item $3$ below holds for every $k \ge 3$. The upper bound given in
  item $1$ holds for every $k$. For item 2, it is not hard to show that $\ch_{n,k}$ can be learnt using a sample of $\Omega \left(\frac{n^k}{\epsilon^2}\right)$ examples by a naive improper learning algorithm, similar to the algorithm we describe in this section for $k=3$.}
the following:
\begin{enumerate}
\item Ignoring computational issues, it is possible to learn the class
  $\ch_{n,3}$ using $O\left(\frac{n}{\epsilon^2}\right)$ examples.
\item It is also possible to \emph{efficiently} learn $\ch_{n,3}$ if we are
provided with a larger training set (of size
$\tilde{\Omega}\left(\frac{n^2}{\epsilon^2}\right)$). This is
  formalized in Theorem \ref{thm:main}.
\item It is impossible to \emph{efficiently} learn $\ch_{n,3}$, if we
are only provided with a training set of size $O\left(\frac{n}{\epsilon^2}\right)$ under Feige's
assumption regarding the hardness of refuting random $\mathrm{3CNF}$
formulas \citep{Feige02}. Furthermore, for every $\alpha \in
[0,0.5)$, it is impossible to learn efficiently with a training set of
size $O\left(\frac{n^{1+\alpha}}{\epsilon^2}\right)$ under a stronger hardness assumption. This is formalized in Theorem \ref{thm:learning_H_new}.
\end{enumerate}
A graphical illustration of our main results is given below:
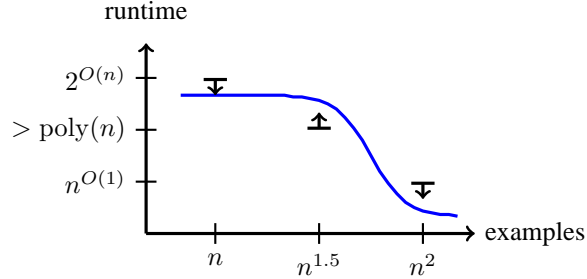
\begin{center}
\begin{tikzpicture}[scale=2.3]

\draw[blue,very thick] (0.20,0.70) --(0.25,0.70) --(0.30,0.70) --(0.35,0.70) --(0.40,0.70) --(0.45,0.70) --(0.50,0.70) --(0.55,0.70) --(0.60,0.70) --(0.65,0.70) --(0.70,0.70) --(0.75,0.70) --(0.80,0.70) --(0.85,0.69) --(0.90,0.69) --(0.95,0.68) --(1.00,0.67) --(1.05,0.65) --(1.10,0.62) --(1.15,0.57) --(1.20,0.51) --(1.25,0.44) --(1.30,0.35) --(1.35,0.26) --(1.40,0.19) --(1.45,0.13) --(1.50,0.08) --(1.55,0.05) --(1.60,0.03) --(1.65,0.02) --(1.70,0.01) --(1.75,0.01) --(1.80,0.00) ;

\draw[|->,very thick,black] (0.4,0.8) -- (0.4,0.7);
\draw[|->,very thick,black] (1,0.5) -- (1,0.6);
\draw[|->,very thick,black] (1.6,0.2) -- (1.6,0.1);

\draw[->,very thick,black] (0,-0.1) --  (0,1) node[above=5pt]
{runtime};
\draw[-,thick,black] (0.06,0.8) -- (-0.06,0.8) node[left] {$2^{O(n)}$};
\draw[-,thick,black] (0.06,0.5) -- (-0.06,0.5) node[left] {$>\poly(n)$};
\draw[-,thick,black] (0.06,0.2) -- (-0.06,0.2) node[left] {$n^{O(1)}$};
\draw[->,very thick,black] (0,-0.1) --  (1.9,-0.1) node[right] {examples};
\draw[-,thick,black] (1.6,-0.04) -- (1.6,-0.16) node[below] {$n^2$};
\draw[-,thick,black] (1,-0.04) -- (1,-0.16) node[below] {$n^{1.5}$};
\draw[-,thick,black] (0.4,-0.04) -- (0.4,-0.16)  node[below] {$n$};
\end{tikzpicture}
\end{center}
\vskip -0.4cm

The proof of item 1 above is easy -- simply note that $H_{n,3}$ has 
VC dimension $n+1$. 

Item 2 is proved in section \ref{sec:upper_bounds}, relying on the
results of \cite{HazanKaSh12}. We note, however, that a weaker result,
that still suffices for answering Question 1 in the affirmative, can be
proven using a naive improper learning algorithm.  In particular, we
show below how to learn $\ch_{n,3}$ efficiently with a sample of
$\Omega\left(\frac{n^3}{\epsilon^2}\right)$ examples. The idea is to
replace the class $\ch_{n,3}$ with the class $\{\pm 1\}^{C_{n,3}}$
containing \emph{all} functions from $C_{n,3}$ to $\{\pm
1\}$. Clearly, this class contains $H_{n,3}$. In addition, we can
efficiently find a function $f$ that minimizes the empirical training
error over a training set $S$ as follows: For every $x \in C_{n,k}$,
if $x$ does not appear at all in the training set we will set $f(x)$
arbitrarily to $1$. Otherwise, we will set $f(x)$ to be the majority
of the labels in the training set that correspond to $x$. Finally,
note that the VC dimension of $\{\pm 1\}^{C_{n,3}}$ is smaller than
$n^3$ (since $|C_{n,3}| < n^3$). Hence, standard generalization
results (e.g. \cite{Vapnik95}) implies that a training set size of
$\Omega\left(\frac{n^3}{\epsilon^2}\right)$ suffices for learning this
class.

Item 3 is shown in section \ref{sec:main} by presenting a novel
technique for establishing statistical-computational tradeoffs.

{\bf The class $\ch_{n,2}$.} Our main result gives a positive answer
to Question 1 for the task  of improperly learning $\ch_{n,k}$ for
$k\ge 3$. A natural question is what happens for $k=2$ and
$k=1$. Since $\VC(\ch_{n,1})= \VC(\ch_{n,2})=n+1$, the information
theoretic barrier for learning these classes is
$\Theta\left(\frac{n}{\epsilon^2}\right)$. In section
\ref{sec:upper_bounds}, we prove that $\ch_{n,2}$ (and, consequently,
$\ch_{n,1}\subset \ch_{n,2}$) can be learnt using
$O\left(\frac{n\log^3(n)}{\epsilon^2}\right)$ examples, indicating
that significant computational-statistical tradeoffs start to 
manifest themselves only for $k\ge 3$.

\subsection{Previous approaches, difficulties, and our techniques}
\citep{DecaturGoRo98} and \citep{Servedio00} gave positive answers to
Question 1 in the realizable PAC learning model. Under cryptographic
assumptions, they showed that there exist binary learning problems, in
which more data can provably be used to speed up training
time. \citep{ShalevShTr12} showed a similar result for the agnostic
PAC learning model.  In all of these papers, the main idea is to
construct a hypothesis class based on a one-way function. However, the
constructed classes are of a very synthetic nature, and are of almost
no practical interest. This is mainly due to the construction
technique which is based on one way functions.  In this work, instead
of using cryptographic assumptions, we rely on the hardness of refuting
random $\mathrm{3CNF}$ formulas. The simplicity and
flexibility of $\mathrm{3CNF}$ formulas enable us to
derive lower bounds for natural classes such as halfspaces.

Recently, \citep{BerthetRi13} gave a positive answer to Question 1 in
the context of unsupervised learning. Concretely, they studied the
problem of sparse PCA, namely, finding a \emph{sparse} vector that
maximizes the variance of an unsupervised data. Conditioning on the
hardness of the planted clique problem, they gave a positive answer to
Question 1 for sparse PCA. Our work, as well as the previous work of
\cite{DecaturGoRo98,Servedio00,ShalevShTr12}, studies Question 1 in
the supervised learning setup.  We emphasize that unsupervised
learning problems are radically different than supervised learning
problems in the context of deriving lower bounds. The main reason for
the difference is that in supervised learning problems, the learner is
allowed to employ improper learning, which gives it a lot of power in
choosing an adequate representation of the data. For example, the
upper bound we have derived for the class of sparse halfspaces
switched from representing hypotheses as halfspaces to representation
of hypotheses as tables over $C_{n,3}$, which made the learning
problem easy from the computational perspective. The crux of the
difficulty in constructing lower bounds is due to this freedom of
the learner in choosing a convenient representation. This difficulty
does not arise in the problem of sparse PCA detection, since there the
learner must output a good sparse vector. Therefore, it is not clear
whether the approach given in \citep{BerthetRi13} can be used to
establish computational-statistical gaps in
supervised learning problems.

\section{Background and notation}

For hypothesis class $\ch\subset \{\pm 1\}^{X}$ and a set $Y\subset X$, we define the {\em restriction of $\ch$ to $Y$} by
$\ch|_{Y}=\{h|_Y\mid h\in \ch\}$. 
We denote by $J=J_n$ the all-ones $n\times n$ matrix. We denote the $j$'th vector in the standard basis of $\mathbb R^n$ by $e_j$. 
\subsection{Learning Algorithms}
For $h:C_{n,3}\to\{\pm 1\}$ and a distribution $\cd$ on $C_{n,3}\times\{\pm 1\}$ we denote the {\em error of $h$ w.r.t. $\cd$} by
$\Err_\cd(h)=\Pr_{(x,y)\sim \cd}\left(h(x)\ne y\right)$. For $\ch\subset \{\pm 1\}^{C_{n,3}}$ we denote the {\em error of $\ch$ w.r.t. $\cd$} by $\Err_{\cd}(\ch)=\min_{h\in\ch}\Err_{\cd}(h)$. For a sample $S\in \left(C_{n,3}\times \{\pm 1\}\right)^m$ we denote by
$\Err_{S}(h)$ (resp. $\Err_{S}(\ch)$) the error of $h$ (resp. $\ch$) w.r.t. the empirical distribution induces by the sample $S$.

A {\em learning algorithm}, $L$, receives a sample $S\in
\left(C_{n,3}\times \{\pm 1\}\right)^m$ and return a hypothesis
$L(S):C_{n,3}\to \{\pm 1\}$.  We say that $L$ {\em learns $\ch_{n,3}$
  using $m(n,\epsilon)$ examples} if,\footnote{For simplicity, we
  require the algorithm to succeed with probability of at least
  $9/10$. This can be easily amplified to probability of at least
  $1-\delta$, as in the usual definition of agnostic PAC learning,
  while increasing the sample complexity by a factor of
  $\log(1/\delta)$. } for every distribution $\cd$ on $C_{n,3}\times
\{\pm 1\}$ and a sample $S$ of more than $m(n,\epsilon)$
i.i.d. examples drawn from $\cd$,
\[
\Pr_{S}\left(\Err_{\cd}(L(S))>\Err_{\cd}(\ch_{3,n})+\epsilon\right)<\frac{1}{10}
\]
The algorithm $L$ is {\em efficient} if
it runs in polynomial time in the sample size and returns a hypothesis
that can be evaluated in polynomial time. 

\subsection{Refuting random $\mathrm{3SAT}$ formulas}
We frequently view a boolean assignment to variables $x_1,\ldots,x_n$ as a vector in $\mathbb R^n$. It is convenient, therefore, to assume that boolean variables take values in $\{\pm 1\}$ and
to denote negation by $``-"$ (instead of
the usual $``\neg"$). An $n$-variables $\mathrm{3CNF}$ clause is a boolean formula of the form 
\[
C(x)=(-1)^{j_1}x_{i_1}\vee (-1)^{j_2}x_{i_2}\vee (-1)^{j_1}x_{i_3},\;\;x\in \{\pm 1\}^n
\] 
An $n$-variables $\mathrm{3CNF}$ formula is a boolean formula of the form 
\[
\phi(x)=\wedge_{i=1}^m C_i(x) ~,
\] 
where every $C_i$ is a $\mathrm{3CNF}$ clause. Define the {\em value}, $\val (\phi)$, of $\phi$ as the maximal fraction of clauses that can be simultaneously satisfied. If $\val(\phi)=1$, we say the $\phi$ is {\em satisfiable}.
By $\mathrm{3CNF}_{n,m}$ we denote the set of $\mathrm{3CNF}$ formulas with $n$ variables and $m$ clauses.

Refuting random $\mathrm{3CNF}$ formulas has been studied extensively
(see e.g. a special issue of TCS \cite{DubiosMoSeZe}). It is known that for large enough
$\Delta$ ($\Delta=6$ will suffice) a random formula in
$\mathrm{3CNF}_{n,\Delta n}$ is not satisfiable with probability $1-o(1)$. Moreover, for every $0\le \epsilon<\frac{1}{4}$, and a large enough $\Delta=\Delta(\epsilon)$, the value of a random formula $\mathrm{3CNF}_{n,\Delta n}$ is $\le 1-\epsilon$ with probability $1-o(1)$.

The problem of refuting random $\mathrm{3CNF}$ concerns efficient
algorithms that provide a proof that a random $\mathrm{3CNF}$ is not
satisfiable, or far from being satisfiable. This can be thought of as a game between an adversary and an algorithm. The adversary should produce a $\mathrm{3CNF}$-formula. It can either produce a satisfiable formula, or, produce a formula uniformly at random. The algorithm should identify whether the produced formula is random or satisfiable.

Formally, let $\Delta:\mathbb N\to\mathbb N$ and $0\le \epsilon <\frac{1}{4}$. We say that an efficient algorithm, $A$, {\em $\epsilon$-refutes random $\mathrm{3CNF}$ with ratio $\Delta$} if its input is $\phi\in \mathrm{3CNF}_{n,n\Delta(n)}$, its output is either $\mathrm{``typical"}$ or $\mathrm{``exceptional"}$ and it satisfies:
\begin{itemize}
\item {\em Soundness:} If $\val(\phi)\ge 1-\epsilon$, then
\[
\Pr_{\text{Rand. coins of }A}\left(A(\phi)=\mathrm{``exceptional"}\right) \ge \frac{3}{4}
\]
\item {\em Completeness:}  For every $n$,
\[
\Pr_{\text{Rand. coins of }A,\;\phi\sim \mathrm{Uni}(\mathrm{3CNF}_{n,n\Delta (n)})}\left(A(\phi)=\mathrm{``typical"}\right) \ge 1-o(1)
\]
\end{itemize}
By a standard repetition argument, the probability of $\frac{3}{4}$ can
be amplified to $1-2^{-n}$, while efficiency is preserved. Thus, given
such an (amplified) algorithm, if
$A(\phi)=``\mathrm{typical}"$, then with confidence of $1-2^{-n}$ we
know that $\val(\phi)< 1-\epsilon$. Since for random $\phi\in \mathrm{3CNF}_{n,n \Delta (n)}$,
$A(\phi)=``\mathrm{typical}"$ with probability $1-o(1)$, such an algorithm provides, for most $\mathrm{3CNF}$ formulas a proof that their value is less that $1-\epsilon$.

Note that an algorithm that $\epsilon$-refutes random $\mathrm{3CNF}$ with ratio $\Delta$ also $\epsilon'$-refutes random $\mathrm{3CNF}$ with ratio $\Delta$ for every $0\le \epsilon'\le\epsilon$. Thus, the task of refuting random $\mathrm{3CNF}$'s gets easier as $\epsilon$ gets smaller.
Most of the research concerns the case $\epsilon=0$. Here, it is not hard to see that the task is getting easier as $\Delta$ grows.
The best known algorithm \citep{FeigeOf07} $0$-refutes random $\mathrm{3CNF}$ with ratio $\Delta (n)=\Omega(\sqrt{n})$. In \cite{Feige02} it was conjectured that for constant $\Delta$ no efficient algorithm can provide a proof that a random $\mathrm{3CNF}$ is not satisfiable:

\begin{conjecture}[R3SAT hardness assumption -- \citep{Feige02}]\label{hyp:feige_basic} 
For every $\epsilon>0$ and for every large enough integer
$\Delta>\Delta_0(\epsilon)$ there exists no efficient algorithm that $\epsilon$-refutes random $\mathrm{3CNF}$ formulas with ratio $\Delta$.
\end{conjecture}
In fact, for all we know, the following conjecture may be true for every $0 \le \mu \le 0.5$.
\begin{conjecture}[$\mu$-R3SAT hardness assumption]\label{hyp:mu_conj}
For every $\epsilon>0$ and for every integer
$\Delta>\Delta_0(\epsilon)$ there exists no efficient algorithm that $\epsilon$-refutes random $\mathrm{3CNF}$ with ratio $\Delta\cdot n^\mu$.
\end{conjecture}
Note that Feige's conjecture is equivalent to the $0$-R3SAT hardness assumption.

\section{Lower bounds for learning $\ch_{n,3}$}\label{sec:main}
\begin{theorem}[main]\label{thm:main}
Let $0\le \mu\le 0.5$. If the $\mu$-R3SAT hardness assumption
(conjecture \ref{hyp:mu_conj}) is true, then there exists no efficient
learning algorithm that learns the class $\ch_{n,3}$ using
$O\left(\frac{n^{1+\mu}}{\epsilon^2}\right)$ examples.
\end{theorem}

In the proof of Theorem \ref{thm:main} we rely on the validity of a conjecture,
similar to conjecture \ref{hyp:mu_conj} for $3$-variables majority formulas.
Following an argument from \citep{Feige02} (Theorem~\ref{hyp:feige_maj}) the validity of the conjecture on
which we rely for majority formulas follows the validity of conjecture
\ref{hyp:mu_conj}. 

Define
\[
\forall (x_1,x_2,x_3)\in \{\pm 1\}^3,\;\textrm{MAJ}(x_1,x_2,x_3):=\sign(x_1+x_2+x_3)
\]
An $n$-variables $\textrm{3MAJ}$ clause is a boolean formula of the form 
\[
C(x)=\textrm{MAJ}((-1)^{j_1}x_{i_1}, (-1)^{j_2}x_{i_2},(-1)^{j_1}x_{i_3}),\;\;x\in \{\pm 1\}^n
\] 
An $n$-variables $\textrm{3MAJ}$ formula is a boolean formula of the form 
\[
\phi(x)=\wedge_{i=1}^m C_i(x)
\] 
where the $C_i$'s are $\textrm{3MAJ}$ clauses. By $\textrm{3MAJ}_{n,m}$ we denote the set of $\textrm{3MAJ}$ formulas with $n$ variables and $m$ clauses.

\begin{theorem}[\citep{Feige02}]\label{hyp:feige_maj} 
Let $0\le \mu \le 0.5$. If the $\mu$-R3SAT hardness assumption is true, then for every $\epsilon>0$ and for every large enough integer $\Delta>\Delta_0(\epsilon)$ there exists no efficient algorithm with the following properties.
\begin{itemize}
\item Its input is $\phi\in \textrm{3MAJ}_{n,\Delta n^{1+\mu}}$, and its output is either $\mathrm{``typical"}$ or $\mathrm{``exceptional"}$.
\item If $\val(\phi)\ge \frac{3}{4}-\epsilon$, then
\[
\Pr_{\text{Rand. coins of }A}\left(A(\phi)=\mathrm{``exceptional"}\right) \ge \frac{3}{4}
\]
\item For every $n$,
\[
\Pr_{\text{Rand. coins of }A,\;\phi\sim \mathrm{Uni}(\textrm{3MAJ}_{n,\Delta n^{1+\mu}})}\left(A(\phi)=\mathrm{``typical"}\right) \ge 1-o(1)
\]
\end{itemize}
\end{theorem}
Next, we prove Theorem \ref{thm:main}. In fact, we will prove a
slightly stronger result. Namely, define the subclass $\ch_{n,3}^d\subset \ch_{n,3}$, of homogenous halfspaces with binary weights, given by 
$\ch_{n,3}^d=\left\{h_{w,0}\mid w\in \{\pm 1\}^n\right\}$. As we show, under the $\mu$-R3SAT hardness assumption, it is impossible to efficiently learn this subclass using only
$O\left(\frac{n^{1+\mu}}{\epsilon^2}\right)$ examples.

{\em Proof idea:} We will reduce the task of refuting random $\mathrm{3MAJ}$ formulas with linear number of clauses to the task of (improperly) learning $\ch^d_{n,3}$ with linear number of samples. 
The first step will be to construct a transformation that associates every $\mathrm{3MAJ}$ clause with two examples in $C_{n,3}\times \{\pm 1\}$, and every assignment with a hypothesis in $\ch^d_{n,3}$. As we will show, the hypothesis corresponding to an assignment $\psi$ is correct on the two examples corresponding to a clause $C$ if and only if $\psi$ satisfies $C$. 
With that interpretation at hand, every $\mathrm{3MAJ}$ formula $\phi$
can be thought of as a distribution $\cd_\phi$ on $C_{n,3}\times \{\pm 1\}$, which is the empirical distribution induced by $\psi$'s clauses.  It holds furthermore that $\Err_{\cd_\phi}(\ch^d_{n,3})=1-\val(\phi)$.

Suppose now that we are given an efficient learning algorithm for $\ch^d_{n,3}$, that uses $\kappa \frac{n}{\epsilon^2}$ examples, for some $\kappa>0$.
To construct an efficient algorithm for refuting
$\mathrm{3MAJ}$-formulas, we simply feed the learning algorithm
with $\kappa \frac{n}{0.01^2}$ examples drawn from $\cd_{\phi}$ and answer ``exceptional'' if the error of the hypothesis returned by
the algorithm is small. 
If $\phi$ is (almost) satisfiable, the algorithm is guaranteed to
return a hypothesis with a small error. 
On the other hand, if $\phi$ is far from being satisfiable,
$\Err_{\cd_\phi}(\ch^d_{n,3})$ is large. 
If the learning algorithm is
proper, then it must return a hypothesis from $\ch^d_{n,3}$ and
therefore it would necessarily return a hypothesis with a large
error. This argument can be used to show that, unless $NP=RP$,
learning $\ch^d_{n,3}$ with a {\em proper} efficient algorithm is
impossible. However, here we want to rule out {\em improper}
algorithms as well.

The crux of the construction is that if $\phi$
is random, {\em no algorithm} (even improper and even inefficient) can
return a hypothesis with a small error. The reason for that is that since the sample provided to the algorithm consists of only $\kappa \frac{n}{0.01^2}$ samples, the algorithm won't see most of $\psi$'s clauses, and, consequently, the produced hypothesis $h$ will be {\em independent of them}. Since these clauses are random, $h$ is likely to err on about half of them, so that $\Err_{D_{\phi}}(h)$ will be close to half!

To summarize we constructed an efficient algorithm with the following
properties: if $\phi$ is almost satisfiable, the algorithm will return
a hypothesis with a small error, and then we will declare
``exceptional'', while for random $\phi$, the
algorithm will return a hypothesis with a large error, and we will
declare ``typical''. 

Our construction crucially relies on the restriction to learning
algorithm with a small sample complexity. Indeed, if the learning
algorithm obtains more than $n^{1+\mu}$ examples, then it will see
most of $\psi$'s clauses, and therefore it might succeed in
``learning'' even when the source of the formula is random. Therefore,
we will declare ``exceptional'' even when the source is random.

\begin{proof} (of theorem \ref{thm:main})
Assume by way of contradiction that the $\mu$-R3SAT hardness
assumption is true and yet there exists an efficient learning algorithm
that learns the class $\ch_{n,3}$ using
$O\left(\frac{n^{1+\mu}}{\epsilon^2}\right)$ examples. Setting
$\epsilon=\frac{1}{100}$, we conclude that there exists an efficient
algorithm $L$ and a constant $\kappa>0$ such that given a sample $S$ of more than $\kappa\cdot n^{1+\mu}$ examples drawn from a distribution $\cd$ on $C_{n,3}\times \{\pm 1\}$, returns a classifier $L(S):C_{n,3}\to \{\pm 1\}$ such that
\begin{itemize}
\item $L(S)$ can be evaluated efficiently.
\item W.p. $\ge\frac{3}{4}$ over the choice of $S$,
  $\Err_{\cd}(L(S))\le \Err_{\cd}(\ch_{n,3})+\frac{1}{100}$. 
\end{itemize}  
Fix $\Delta$ large enough such that $\Delta>100 \kappa$ and the conclusion of Theorem \ref{hyp:feige_maj} holds with $\epsilon=\frac{1}{100}$.
We will construct an algorithm, $A$, contradicting Theorem \ref{hyp:feige_maj}. On input $\phi\in \textrm{3MAJ}_{n,\Delta n^{1+\mu}}$ consisting of the $\textrm{3MAJ}$ clauses $C_1,\ldots, C_{\Delta n^{1+\mu}}$, the algorithm $A$ proceeds as follows
\begin{enumerate}
\item Generate a sample $S$ consisting of $\Delta n^{1+\mu}$ examples as follows. For every clause, $C_k=\textrm{MAJ}((-1)^{j_1}x_{i_1},(-1)^{j_2}x_{i_2},(-1)^{j_3}x_{i_3})$, generate an example $(x_k,y_k)\in C_{n,3}\times \{\pm 1\}$ by choosing $b\in \{\pm 1\}$ at random and letting
$$(x_k,y_k)=b\cdot\left(\sum_{l=1}^3 (-1)^{j_l}e_{i_l},1\right)\in C_{n,3}\times \{\pm 1\}~.$$
For example, if $n=6$, the clause is $\textrm{MAJ}(- x_2,x_3,x_6)$ and $b=-1$, we generate the example
$$\left((0,1,-1,0,0,-1),-1 \right)$$
\item Choose a sample $S_1$ consisting of $\frac{\Delta n^{1+\mu}}{100}\ge \kappa\cdot n^{1+\mu}$ examples by choosing at random (with repetitions) examples from $S$.
\item Let $h=L(S_1)$. If $\Err_{S}(h)\le \frac{3}{8}$, return $\mathrm{``exceptional"}$. Otherwise, return $\mathrm{``typical"}$.
\end{enumerate}

We claim that $A$ contradicts Theorem \ref{hyp:feige_maj}. Clearly, $A$ runs in polynomial time. It remains to show that 
\begin{itemize}
\item If $\val(\phi)\ge \frac{3}{4}-\frac{1}{100}$, then
\[
\Pr_{\text{Rand. coins of }A}\left(A(\phi)=\mathrm{``exceptional"}\right) \ge \frac{3}{4}
\]
\item For every $n$,
\[
\Pr_{\text{Rand. coins of }A,\;\phi\sim \mathrm{Uni}(\textrm{3MAJ}_{n,\Delta n^{1+\mu}})}\left(A(\phi)=\mathrm{``typical"}\right) \ge 1-o(1)
\]
\end{itemize}

Assume first that $\phi\in \textrm{3MAJ}_{n,\Delta n^{1+\mu}}$ is chosen at random. Given the sample $S_1$, the sample $S_2:=S\setminus S_1$ is a sample of $|S_2|$ i.i.d. examples which are independent from the sample $S_1$, 
and hence also from $h=L(S_1)$. Moreover, for every example $(x_k,y_k)\in S_2$, $y_k$ is a Bernoulli random variable with parameter $\frac{1}{2}$ which is independent of $x_k$. To see that, note that an example whose instance is $x_k$ can be generated by exactly two clauses -- one corresponds to $y_k=1$, while the other corresponds to $y_k=-1$ (e.g., the instance $(1,-1,0,1)$ can be generated from the clause $\textrm{MAJ}(x_1,-x_2,x_4)$ and $b=1$ or the clause $\textrm{MAJ}(-x_1,x_2,-x_4)$ and $b=-1$). Thus, given the instance $x_k$, the probability that $y_k=1$ is $\frac{1}{2}$, independent of $x_k$.

It follows that $\Err_{S_2}(h)$ is an average of at least $\left(1-\frac{1}{100}\right)\Delta n^{1+\mu}$ independent Bernoulli random variable. By 
Chernoff's bound, with probability $\ge 1-o(1)$, $\Err_{S_2}(h)>\frac{1}{2}-\frac{1}{100}$. Thus,
$$\Err_{S}(h)\ge \left(1-\frac{1}{100}\right)\Err_{S_2}(h)\ge \left(1-\frac{1}{100}\right)\cdot \left(\frac{1}{2}-\frac{1}{100}\right)>\frac{3}{8}$$
And the algorithm will output $\mathrm{``typical"}$.

Assume now that $\val(\phi)\ge\frac{3}{4}-\frac{1}{100}$ and let $\psi\in\{\pm 1\}^n$ be an assignment that indicates that. Let $\Psi\in \ch_{n,3}$ be the hypothesis $\Psi(x)=\sign\left(\inner{\psi,x}\right)$. It can be easily checked that $\Psi(x_k)=y_k$ if and only if $\psi$ satisfies $C_k$. Since $\val(\phi)\ge\frac{3}{4}-\frac{1}{100}$, it follows that
$$\Err_{S}(\Psi)\le \frac{1}{4}+\frac{1}{100} ~.$$
Thus,
$$\Err_{S}(\ch_{n,3})\le \frac{1}{4}+\frac{1}{100} ~.$$
By the choice of $L$, with probability $\ge 1-\frac{1}{4}=\frac{3}{4}$, $$\Err_S(h)\le \frac{1}{4}+\frac{1}{100}+\frac{1}{100}< \frac{3}{8}$$
and the algorithm will return $\mathrm{``exceptional"}$.
\end{proof}

\section{Upper bounds for learning $\ch_{n,2}$ and $\ch_{n,3}$}\label{sec:upper_bounds}

The following theorem derives upper bounds for learning $\ch_{n,2}$
and $\ch_{n,3}$. Its proof relies on results from \cite{HazanKaSh12}
about learning $\beta$-decomposable matrices, and due to the lack of
space is given in the appendix. 
\begin{theorem}\label{thm:learning_H_new} ~
\begin{itemize}
\item There exists an efficient algorithm that learns $\ch_{n,2}$ using
$O\left(\frac{n\log^3(n)}{\epsilon^2}\right)$ examples 
\item There exists an efficient algorithm that learns $\ch_{n,3}$ using
$O\left(\frac{n^2\log^3(n)}{\epsilon^2}\right)$ examples 
\end{itemize}
\end{theorem}

\section{Discussion}

We formally established a computational-sample complexity tradeoff for
the task of (agnostically and improperly) PAC learning of halfspaces
over $3$-sparse vectors. Our proof of the lower bound relies on a
novel, non cryptographic, technique for establishing such
tradeoffs. We also derive a new non-trivial upper bound for this task.

{\bf Open questions.} An obvious open question is to close the gap between the lower and upper bounds. We conjecture that $\ch_{n,3}$ can be learnt efficiently using a sample of $\tilde{O}\left(\frac{n^{1.5}}{\epsilon^2}\right)$ examples. Also, we believe that our new proof
technique can be used for establishing computational-sample
complexity tradeoffs for other natural learning problems.

\paragraph{Acknowledgements:}
Amit Daniely is a recipient of the Google Europe Fellowship in Learning Theory, and this research is supported in part by this Google Fellowship. Nati Linial is supported by grants from ISF, BSF and I-Core. Shai Shalev-Shwartz is supported by the Israeli Science Foundation grant number 590-10.

\bibliographystyle{plainnat}
\bibliography{bib}

\newpage

\appendix

\section{Proof of Theorem \ref{thm:learning_H_new}} \label{sec:proof_upper_bounds}

The proof of the theorem relies on results from \cite{HazanKaSh12} about learning $\beta$-decomposable matrices. Let $W$ be an $n\times m$ matrix. We define the {\em symmetrization} of $W$ to be the $(n+m)\times (n+m)$ matrix
\[
\sym(W)=\begin{bmatrix}
0 & W\\ W^T & 0
\end{bmatrix}
\]
We say that $W$ is {\em $\beta$-decomposable} if there exist positive semi-definite matrices $P,N$ for which
\begin{eqnarray*}
\sym(W) &=& P-N
\\
\forall i, P_{ii},N_{ii}&\le& \beta
\end{eqnarray*}
Each matrix in $\{\pm 1\}^{n\times m}$ can be naturally interpreted as
a hypothesis on $[n]\times [m]$.

We say that a learning algorithm $L$ learns a class $\ch_n\subset \{\pm 1\}^{X_n}$ using $m(n,\epsilon,\delta)$ examples if, for every distribution $\cd$ on $X_n\times
\{\pm 1\}$ and a sample $S$ of more than $m(n,\epsilon,\delta)$
i.i.d. examples drawn from $\cd$,
\[
\Pr_{S}\left(\Err_{\cd}(L(S))>\Err_{\cd}(\ch_{n})+\epsilon\right)<\delta
\]

\cite{HazanKaSh12} have
proved\footnote{The result of \cite{HazanKaSh12} is more general than
  what is stated here. Also, \cite{HazanKaSh12} considered the online
  scenario. The result for the statistical scenario, as stated here,
  can be derived by applying standard online-to-batch conversions (see for
  example \cite{CesaBianchiCoGe01}).} that

\begin{theorem}\cite{HazanKaSh12}\label{thm:decomposable}
The hypothesis class of $\beta$-decomposable $n\times m$ matrices with $\pm 1$ entries ban be efficiently learnt using a sample of $O\left( \frac{\beta^2 (n+m) \log(n+m)+\log(1/\delta)}{\epsilon^2}\right)$ examples.
\end{theorem}

We start with a generic reduction from a problem of learning a class
$\cg_n$ over an instance space $X_n\subset \{-1,1,0\}^n$ to the
problem of learning $\beta(n)$-decomposable matrices. We say that $\cg_n$ is {\em realized} by $m_n\times m_n$ matrices that are $\beta(n)$-decomposable if there exists a mapping $\psi_n:X_n\to [m_n]\times [m_n]$ such that for every $h\in\cg_n$ there exists a $\beta(n)$-decomposable $m_n\times m_n$ matrix $W$ for which $\forall x\in X_n,\;h(x)=W_{\psi_n(x)}$. The mapping $\psi_n$ is called a {\em realization} of $\cg_n$.
In the case that the mapping $\psi_n$ can be computed in time polynomial in $n$, we say that $\cg_n$ is {\em efficiently realized} and $\psi_n$ is an {\em efficient realization}. 
It follows from Theorem \ref{thm:decomposable} that:
\begin{corollary}\label{cor:decomposable}
If $\cg_n$ is efficiently realized by $m_n\times m_n$ matrices that
are $\beta(n)$-decomposable then $\cg_n$ can be efficiently learnt
using a sample of $O\left(\frac{\beta(n)^2
    m_n\log(m_n)+\log(1/\delta)}{\epsilon^2}\right)$ examples.
\end{corollary}

We now turn to the proof of Theorem \ref{thm:learning_H_new}. We start with the first assertion, about learning $\ch_{n,2}$.
The idea will be to partition the instance space into a disjoint union
of subsets and show that the restriction of the hypothesis class to
each subset can be efficiently realized by $\beta(n)$-decomposable. 
Concretely, we decompose $C_{n,2}$ into a disjoint union of five sets
\[
C_{n,2}=\cup_{r=-2}^2A^r_{n}
\]
where
\[
A^r_n =\left\{x\in C_{n,2} \mid \sum_{i=1}^n x_i=r\right\} .
\]
In section \ref{sec:learning_A} we will prove that
\begin{lemma}\label{lem:learning_A} 
For every $-2\le r\le 2$, $\ch_{n,2}|_{A^r_n}$ can be efficiently realized by $n\times n$ matrices that are $O(\log(n))$-decomposable.
\end{lemma}
To glue together the five restrictions, we will rely on the following Lemma, whose proof is given in section \ref{sec:learning_A}. 

\begin{lemma}\label{lem:decompostion}
  Let $X_1,...,X_k$ be partition of a domain $X$ and let $H$ be a
  hypothesis class over $X$.  Define $H_i=H|_{X_i}$. Suppose the for every $H_i$ there exist a learning algorithm that learns $H_i$ using $\le C(d+\log(1/\delta))/\epsilon^2$ examples, for some constant $C \ge
  8$. Consider the algorithm $A$ which receives an i.i.d. training set
  $S$ of $m$ examples from $X \times \{0,1\}$ and applies the learning
  algorithm for each $H_i$ on the examples in $S$ that belongs to
  $X_i$. Then, $A$ learns $H$ using at most
\[
 \frac{2Ck(d +
\log(2k/\delta))}{ \epsilon^2}
\]
examples.
\end{lemma}
The first part of Theorem \ref{thm:learning_H_new} is therefore follows from
Lemma \ref{lem:learning_A}, Lemma \ref{lem:decompostion} and Corollary \ref{cor:decomposable}.

Having the first part of Theorem \ref{thm:learning_H_new} and Lemma \ref{lem:decompostion} at hand, it is not hard to prove the second part of Theorem \ref{thm:learning_H_new}:

For $1\le i\le n-2$ and $b\in \{\pm 1\}$ define
\[
D_{n,i,b}=\left\{x\in C_{n,3}\mid x_i=b\text{ and }\forall j<i,\;x_j=0\right\}
\]

Let $\psi_n:C_{n,3}\to C_{n,2}$ be the mapping that zeros the first non zero coordinate. It is not hard to see that $\ch_{n,3}|_{D_{n,i,b}}=\left\{h\circ\psi_n|_{D_{n,i,b}} \mid h\in \ch_{n,2} \right\}$. Therefore $\ch_{n,3}|_{D_{n,i,b}}$ can be identified with $\ch_{n,2}$ using the mapping $\psi_n$, and therefore can efficiently learnt using $O\left(\frac{n\log^3(n)+\log(1/\delta)}{\epsilon^2}\right)$ examples (the dependency on $\delta$ does not appear in the statement, but can be easily inferred from the proof). The second part of Theorem \ref{thm:learning_H_new} is therefore follows from the first part of the Theorem and Lemma \ref{lem:decompostion}.
 
\subsection{Proofs of Lemma \ref{lem:learning_A} and Lemma \ref{lem:decompostion}}\label{sec:learning_A}
In the proof, we will rely on the following facts. The {\em tensor product} of two matrices $A\in M_{n\times m}$ and $B\in M_{k\times l}$ is defined as the $(n\cdot k)\times (m\cdot l)$ matrix
\[
A\otimes B=\begin{bmatrix}
A_{1,1}\cdot B &\cdots& A_{1,m}\cdot B
\\
\vdots & \ddots & \vdots
\\
A_{n,1}\cdot B &\cdots& A_{m,m}\cdot B
\end{bmatrix}
\]
\begin{proposition}\label{prop:lean_A_1}
Let $W$ be a $\beta$-decomposable matrix and let $A$ be a PSD matrix whose diagonal entries are upper bounded by $\alpha$. Then $W\otimes A$ is $(\alpha\cdot \beta )$-decomposable.
\end{proposition}
\begin{proof}
It is not hard to see that for every matrix $W$ and a symmetric matrix $A$,
\[
\sym(W)\otimes A=\sym(W\otimes A)
\]
Moreover, since the tensor product of two PSD matrices is PSD, if $\sym(W)=P-N$ is a $\beta$-decomposition of $W$, then
\[
\sym(W\otimes A)=P\otimes A-N\otimes A
\]
is a $(\alpha\cdot \beta )$-decomposition of $W\otimes A$.
\end{proof}

\begin{proposition}\label{prop:lean_A_2}
If $W$ is a $\beta$-decomposable matrix, then so is every matrix obtained from $W$ by iteratively deleting rows and columns.
\end{proposition}
\begin{proof}
It is enough to show that deleting one row or column leaves $W$ $\beta$-decomposable.
Suppose that $W'$ is obtained from $W\in M_{n\times m}$ by deleting the $i$'th row (the proof for deleting columns is similar). It is not hard to see that $\sym(W')$ is the $i$'th principal minor of $\sym(W)$. Therefore, since principal minors of PSD matrices are PSD matrices as well, if $\sym(W)=P-N$ is $\beta$-decomposition of $W$ then $\sym(W')=[P]_{i,i}-[N]_{i,i}$ is a $\beta$-decomposition of $W'$.
\end{proof}

\begin{proposition}\label{prop:lean_A_3}\cite{HazanKaSh12}
Let $T_n$ be the upper triangular matrix whose all entries in the diagonal and above are $1$, and whose all entries beneath the diagonal are $-1$. Then $T_n$ is $O(\log(n))$-decomposable.
\end{proposition}
Lastly, we will also need the following generalization of proposition \ref{prop:lean_A_3}
\begin{proposition}\label{prop:lean_A_4}
Let $W$ be an $n\times n$ $\pm 1$ matrix. Assume that there exists a sequence $0\le j(1),\ldots, j(n)\le n$ such that
\[
W_{ij}=\begin{cases} -1 & j\le j(i) \\ 1 & j> j(i)\end{cases}
\]
Then, $W$ is $O(log(n))$-decomposable.
\end{proposition}
\begin{proof}
Since switching rows of a $\beta$-decomposable matrix leaves a $\beta$-decomposable matrix, we can assume without loss of generality that $j(1)\le j(2)\le\ldots\le j(n)$.
Let $J$ be the $n\times n$ all ones matrix. It is not hard to see that $W$ can be obtained from $T_n\otimes J$ by iteratively deleting rows and columns. Combining propositions \ref{prop:lean_A_1}, \ref{prop:lean_A_2} and \ref{prop:lean_A_3}, we conclude that $W$ is $O(\log(n))$-decomposable, as required.
\end{proof}

We are now ready to prove Lemma \ref{lem:learning_A}

\begin{proof} (of Lemma \ref{lem:learning_A})
Denote $\ca^r_n = \ch_{n,2}|_{A^r_n}$. We split into cases.

{\em Case 1, r=0:}
Note that $A_n^0=\{e_i-e_j\mid i,j\in [n]\}$. Define $\psi_n:A_n^0\to [n]\times [n]$ by $\psi_n(e_i-e_j)=(i,j)$. We claim that $\psi_n$ is an efficient realization of $\ca_n^0$ by $n\times n$ matrices that are $O(\log(n))$ decomposable. Indeed, let $h=h_{w,b}\in \ca_n^0$, and let 
$W$ be the $n\times n$ matrix $W_{ij}=W_{\psi_n(e_i-e_j)}=h(e_i-e_j)$. It is enough to show that $W$ is $O(\log(n))$-decomposable.

We can rename the coordinates so that
\begin{equation}\label{eq:1}
w_1\ge w_2\ge\ldots\ge w_n
\end{equation}
From equation (\ref{eq:1}), it is not hard to see that there exist numbers
\[
0\le j(1)\le j(2)\le\ldots\le j(n)\le n
\]
for which
\[
W_{ij}=\begin{cases} -1 & j\le j(i) \\ 1 & j> j(i)\end{cases}
\] 
The conclusion follows from Proposition \ref{prop:lean_A_4}

{\em Case 2, r=2 and r=-2:}
We confine ourselves to the case $r=2$. The case $r=-2$ is similar. Note that $A_n^2=\{e_i+e_j\mid i\ne j\in [n]\}$. Define $\psi_n:A_n^2\to [n]\times [n]$ by $\psi_n(e_i+e_j)=(i,j)$. We claim that $\psi_n$ is an efficient realization of $\ca_n^2$ by $n\times n$ matrices that are $O(\log(n))$ decomposable. Indeed, let $h=h_{w,b}\in \ca_n^2$, and let 
$W$ be the $n\times n$ matrix $W_{ij}=W_{\psi_n(e_i+e_j)}=h(e_i+e_j)$. It is enough to show that $W$ is $O(\log(n))$-decomposable.

We can rename the coordinates so that
\begin{equation}\label{eq:2}
w_1\le w_2\le\ldots\le w_n
\end{equation}
From equation (\ref{eq:2}), it is not hard to see that there exist numbers
\[
n\ge j(1)\ge j(2)\ge\ldots\ge j(n)\ge 0
\]
for which
\[
W_{ij}=\begin{cases} -1 & j\le j(i) \\ 1 & j> j(i)\end{cases}
\] 
The conclusion follows from Proposition \ref{prop:lean_A_4}

{\em Case 3, r=1 and r=-1:}
We confine ourselves to the case $r=1$. The case $r=-1$ is similar. Note that $A_n^1=\{e_i \mid i \in [n]\}$. Define $\psi_n:A_n^0\to [n]\times [n]$ by $\psi_n(e_i)=(i,i)$. We claim that $\psi_n$ is an efficient realization of $\ca_n^1$ by $n\times n$ matrices that are $3$-decomposable (let alone, $\log (n)$-decomposable). Indeed, let $h=h_{w,b}\in \ca_n^1$, and let 
$W$ be the $n\times n$ matrix with  $W_{ii}=W_{\psi_n(e_i)}=h(e_i)$ and $-1$ outside the diagonal. It is enough to show that $W$ is $3$-decomposable. Since $J$ is $1$-decomposable, it is enough to show that $W+J$ is $2$-decomposable. However, it is not hard to see that every diagonal matrix $D$ is $(\max_i |D_{ii}|)$-decomposable.
\end{proof}

\begin{proof} (of Lemma \ref{lem:decompostion})
Let $S = (x_1,y_1),\ldots,(x_m,y_m)$ be a training set and let
$\hat{m}_i$ be the number of examples in $S$ that belong to $X_i$. 
Given that the values of the random variables $\hat{m}_1,\ldots,\hat{m}_i$ is
determined, we have that w.p. of at least $1- \delta$, 
\[ 
\forall i,~~
\Err_{D_i}(h_i) - \Err_{D_i}(h^*) \le \sqrt{\frac{C(d +
\log(k/\delta))}{\hat{m}_i}} ,
\]
where $D_i$ is the induced distribution over $X_i$, $h_i$ is the
output of the $i$'th algorithm, and $h^*$ is the optimal hypothesis
w.r.t. the original distribution $D$. Define, 
\[
m_i = \max\{C(d+\log(k/\delta)), \hat{m}_i \} ~.
\]
It follows from the above that we also have, w.p. at least $1-\delta$,
for every $i$,
\[
\Err_{D_i}(h_i) - \Err_{D_i}(h^*) \le  \sqrt{\frac{C(d +
\log(k/\delta))}{m_i}} =: \epsilon_i .
\]
Let $\alpha_i = D\{(x,y) : x \in \cx_i\}$, and note that $\sum_i
\alpha_i = 1$. Therefore, 
\begin{align*}
\Err_D(h_S) - \Err_D(h^*) &\le \sum_i \alpha_i \epsilon_i = \sum_i \sqrt{\alpha_i}
\sqrt{\alpha_i \epsilon_i^2} \\
&\le \sqrt{\sum_i \alpha_i} \,
\sqrt{\sum_i \alpha_i \epsilon_i^2} = \sqrt{\sum_i \alpha_i
  \epsilon_i^2} \\
&= 
\sqrt{\frac{C(d +
\log(k/\delta))}{ m}} \sqrt{\sum_i \frac{\alpha_i m}{m_i}} .
\end{align*}
Next note that if $\alpha_i
m< C(d+\log(k/\delta))$ then $\alpha_i m/m_i \le 1$. Otherwise, using
Chernoff's inequality, for every $i$ we have
\[
\Pr[m_i < 0.5\alpha_i m] \le e^{- \alpha_i m / 8}\le e^{- (d+\log(k/\delta)) } = e^{-d}
\frac{\delta}{k} \le \frac{\delta}{k} ~.
\]
Therefore, by the union bound, 
\[
\Pr[ \exists i: m_i < 0.5 \alpha_i m] \le \delta . 
\]
It follows that with probability of at least $1-\delta$, 
\[
\sqrt{\sum_i \frac{\alpha_i m}{m_i}} \le \sqrt{2k} ~.
\]
All in all, we have shown that with probability of at least
$1-2\delta$ it holds that
\[
\Err_D(h_S) - \Err_D(h^*) \le \sqrt{\frac{2Ck(d +
\log(k/\delta))}{ m}} ~.
\]
Therefore, the the algorithm learns $\ch$ using
\[
 \le \frac{2Ck(d +
\log(2k/\delta))}{ \epsilon^2}
\]
examples.
\end{proof}
\end{document}